%% file: main.tex
\documentclass[9pt,twocolumn]{article}
\usepackage{amsmath,amsfonts,amsthm}
\usepackage{paralist}
\usepackage{graphics}
\usepackage{color}
\usepackage{multirow}
\usepackage{multicol}
\usepackage{booktabs}
\usepackage{tikz}
\usepackage{url}

\usepackage{array}
\usepackage{algorithmicx}
\usepackage[noend]{algpseudocode}
\usepackage{algorithm}

\advance\topmargin by -37pt     %
\usetikzlibrary{fit}
\usetikzlibrary{backgrounds}
\usetikzlibrary{shapes}
\usetikzlibrary{matrix}
\usetikzlibrary{patterns}
\usetikzlibrary{calc,arrows,automata}
\tikzstyle{State}=[circle, thick, minimum size=0.6cm, inner sep=0cm,draw=black]
\tikzstyle{BState}=[circle, thick, minimum size=0.8cm, inner sep=0cm,draw=black]
\tikzstyle{RState}=[circle, very thick, minimum size=0.8cm, inner sep=0cm,draw=red]

\tikzstyle{node}=[circle, minimum size=0.2cm, inner sep=.02cm,draw=black]

\newtheorem{lemma}{Lemma}

\newtheorem{definition}{Definition}

\newtheorem{theorem}{Theorem}
\newtheorem{example}{Example}
\newtheorem{remark}{Remark}

\newcommand{\pomdp}{\mathcal{M}}
\newcommand{\straa}{\sigma}

\newcommand{\states}{S}
\newcommand{\acts}{A}
\newcommand{\tran}{\delta}
\newcommand{\obs}{Z}
\newcommand{\obsfunc}{O}
\newcommand{\initdist}{\lambda}
\newcommand{\dist}{\mathcal{D}}
\newcommand{\supp}{\mathit{supp}}
\newcommand{\mmid}{\mathord\mid}
\newcommand{\plays}{\mathit{Run}}
\newcommand{\fpaths}{\mathit{FHist}}
\newcommand{\probm}{\mathbb{P}}
\newcommand{\E}{\mathbb{E}}
\newcommand{\target}{T}
\newcommand{\Reach}{\mathit{Reach}}
\newcommand{\enrg}{E}
\newcommand{\ca}{\mathit{cap}}
\newcommand{\Zset}{\mathbb{Z}}
\newcommand{\Nset}{\mathbb{N}}
\newcommand{\cost}{\mathit{c}}
\newcommand{\tot}{\mathit{TC}}
\newcommand{\elev}{\mathit{EL}}
\newcommand{\Safe}{\mathit{Safe}}
\newcommand{\enup}{\mathit{EnUp}}
\newcommand{\size}[1]{||#1||}
\newcommand{\belsup}{B}
\newcommand{\belsupset}{BS}
\newcommand{\allow}{\mathit{all}}
\newcommand{\energySafe}{\mathsf{EnSafe}}
\newcommand{\optCost}{\mathsf{optCost}}
\newcommand{\val}{\mathsf{Val}}
\newcommand{\length}{\mathit{len}}

\begin{document}

\twocolumn[
 \begin{@twocolumnfalse}

\title{Stochastic Shortest Path with Energy Constraints in POMDPs}
\author{Tom\'a\v{s} Br\'azdil\\
Masaryk University\\
Brno, Czech Republic\\
xbrazdil@fi.muni.cz
\and
Krishnendu Chatterjee\\
IST Austria \\
Klosterneuburg, Austria\\
kchatterjee@ist.ac.at
\and
Martin Chmel\'ik\\
IST Austria \\
Klosterneuburg, Austria\\
mchmelik@ist.ac.at
\and
Anchit Gupta\\
IIT Bombay \\
Mumbai, India\\
anchit@iitb.ac.in
\and
Petr Novotn\'y\\
IST Austria \\
Klosterneuburg, Austria\\
pnovotny@ist.ac.at
}
\date{}

\maketitle

\end{@twocolumnfalse}]

\begin{abstract}
We consider partially observable Markov decision processes (POMDPs) with a set of target states and positive integer costs associated with every transition. The traditional optimization objective (stochastic shortest path) asks to minimize the expected total cost until the target set is reached. We extend the traditional framework of POMDPs to model energy consumption, which represents a hard constraint. There are energy levels that may increase and decrease with transitions, and the hard constraint requires that the energy level must remain positive in all steps till the target is reached. Our contribution is twofold. First, we present a novel algorithm for solving POMDPs with energy levels, developing on existing POMDP solvers and using real-time dynamic programming as its main method. Our second contribution is related to policy representation. For larger POMDP instances the policies computed by existing solvers are too large to be understandable. We present an automated procedure based on machine learning techniques that automatically extracts important decisions of a policy and computes its succinct, human readable representation. Finally, we show experimentally that our algorithm performs well and computes succinct policies on a number of POMDP instances from the literature that were naturally enhanced with energy levels.
\end{abstract}

\section{Introduction}
\label{sec:intro}

Motion and task planning for autonomous agents are one of the classical problems studied in AI and robotics. One of the main challenges that make the problem difficult is the presence of \emph{uncertainty} about the state of the agent and its environment~\cite{KLC:planning}, caused for instance by the agent's unreliable sensors.
To account for these issues, powerful abstract frameworks for solving planning problems under uncertainty were developed, among which the \emph{Partially observable Markov decision processes (POMDPs)} play a crucial role. 

Each POMDP describes a discrete, typically finite-state system that exhibits both probabilistic and non-deterministic behaviour~\cite{Sondik:thesis,PT:complexity-mdp-pomdp}. Probabilities are useful for modelling sensor errors, hardware failures, and similar events whose rate of occurrence can be established empirically, while non-determinism represents the freedom of the agent's controller to choose appropriate control input. The imperfection of agent's sensors is represented by \emph{observations}. In every step the controller receives an observation but not the current state itself. \emph{Policies (policies),} i.e. rules for resolving non-determinism in POMDPs, can be viewed as blueprints for implementing concrete controllers of the agent. Hence, given a POMDP modelling an agent and its environment the usual task is to find a policy ensuring that the behaviour of the system conforms to a given \emph{specification} or \emph{objective}.

Various types of objectives in POMDPs have been studied. Typically it is assumed that there is a reward (resp. cost) function that assigns rewards (resp. costs) to transitions of the system. The goal of the agent is to maximize (resp. minimize) the reward (resp. cost) over a finite-horizon~\cite{SS:pomdp-finite-horizon}, or over an infinite horizon~\cite{Sondik:discounted-infinite}, where the sequence of rewards (resp. costs) can be aggregated by considering the \emph{discounted reward}~\cite{PGT:pomdp03,SS04} or the \emph{average reward}~\cite{Puterman94,FV97}, etc.
Particularly relevant from the planning point of view is the \emph{indefinite-horizon (or stochastic shortest path)} objective~\cite{BG09,Bertsekas95,CCGK:POMDP-cost}, which asks to compute a policy that reaches a state from a given set of target states $T$ and minimizes the expected total cost till the target set $T$ is reached, i.e., the expected sum of costs of all transitions traversed before reaching $T$.
 Typically $T$ is such that reaching a state of $T$ corresponds to the agent completing some assigned task. 


\paragraph{Energy and Soft vs. Hard Constraints} 
Most autonomous robotic devices operate under certain \emph{energy constraints}, i.e. they need a steady supply of some resource (in the form of, e.g. fuel, electricity, etc.) to operate correctly. While the stochastic shortest path (SSP) objective can in principle express specifications of the form "complete the task while minimizing the expected consumption of some resource," this approach is not suitable for modelling of resource-constrained systems, as the SSP objective only talks about the expected cost, without giving any guarantees on the cost of concrete executions of the modelled system, i.e., it is an example of a \emph{soft constraint}. In particular, to use costs in SSP objectives to model resource consumption we would need to assume that the amount of a resource consumed by making a transition (represented by the transition's cost) is always available. This is not always realistic. An autonomous robot typically has a battery of a finite capacity (or a fuel tank of finite volume) which is continually depleted as the robot operates. The total amount of a resource required to complete the task can exceed this capacity, prompting the robot to periodically recharge the battery (or refuel the tank) at special charging points (petrol stations). When the resource is depleted, no action remains available to the robot, i.e., keeping the energy level positive is a \emph{hard constraint} that must hold along every single execution of the system, no matter the outcomes of stochastic choices. (The issue of expectation-based vs. execution-based constraints was already examined in the setting of perfectly observable MDPs, see~\cite{Altman:constrained-mdp,RV:constrained-mdp,RV:constrained-multichain,RC:constrained-telecommunication}.)

In this paper we address this issue and extend POMDPs with \emph{energy constraints}. That is, to a POMDP $\pomdp$ with a given objective we assign a positive integer capacity $\ca$ and  to each transition of $\pomdp$ we assign an integer \emph{update} representing the amount of a resource consumed or reloaded by this transition. Such a POMDP starts with some initial level of a resource (say $\ca$, i.e. the resource is loaded to full capacity) which is then modified as the system evolves: whenever a transition with some update $u$ is traversed, the resource level changes from $\ell$ to $\min\{\ell+u,\ca\}$ (discarding any quantity exceeding $\ca$ captures the fact that the robot's storage capacity cannot be exceeded). The task is to find a policy ensuring the original objective and at the same time ensuring that the resource level stays positive till the target is reached.

\paragraph{Our Results on POMDPs with Energy Constraints}
We study \emph{energy-reachability} problem for POMDPs. In the qualitative version of the problem we ask to find a policy that ensures that the expected total cost is finite before reaching the target state and at the same time keeps the resource level positive. In the quantitative version we additionally seek for a policy that, on top of above two conditions, minimizes the expected total cost till the target is reached. We show how to solve both these problems by reducing them to corresponding problems in POMDPs without energy constraints. In particular, we show that the qualitative energy-reachability problem is \textsc{EXPTIME}-complete, i.e. it has the same complexity as unconstrained qualitative reachability. We experimentally evaluate our approach on standard POMDP models of robot planning.

\paragraph{Representation of Policies} Solving POMDPs with energy-reachability objectives highlights another relevant issue: the representation the computed policies. Policies in POMDPs are often represented in a form of a table~\cite{BG09} or \emph{plan graphs}~\cite{KLC:planning}, which are equivalent to so called finite-memory policies used in verification~\cite{CDH:po-survey}. Size of these structures can become very large and not very readable by humans. For instance, policies ensuring that a target state is reached with probability 1 might require table or plan graph of size exponential in the size of the POMDP~\cite{CDH10a}.

There are two reasons why size and representation of policies matter. First, as offline-computed policies have to be implemented on real-world devices, it is advisable to reduce their memory requirements so that they fit into the device's memory and do not cause delays through inefficient memory access. The second issue, to which we devote a particular attention in this paper, is the one of human readability. From the engineering point of view it is vital to be able to visualize the policy and understand its behaviour. This is reflected in numerous informal rules for safety-critical system design that enforce "simplicity" and "readability"~\cite{Holzmann:NASA-commandments,Spencer:C-commandments} as well as in academic treatments of the subject~\cite[Chapter 2 on "Simplicity"]{Kopetz:realtime}. Although many methods for policy computation in POMDPs produce results that are correct by design, the behaviour induced by the computed policy in an actual device might not be reasonable, due to either using an improper model of the system, or too weak specification that does not rule out all undesirable behaviours. In such a case, comprehension of the computed policy can lead the system designer to refine the model or specification in an appropriate way. Easily understandable descriptions of policies can be also interesting for type approval authorities. 

Readability of policies is relevant for POMDPs in general, but the issue is especially pronounced in energy-constrained POMDPs, as the standard representation does not reveal which decisions depend on states and which depend on current resource level, an information useful for identifying bottlenecks caused by insufficient storage capacity or exploiting the fact that policy's dependency on resource levels might not be complex (e.g. "when low on fuel, go to a gas station").



\paragraph{Our Results on Policy Representation}
We study succinct representation of policies via \emph{decision trees}.\footnote{These should not be confused with \emph{policy trees} that represent a complete behaviour of a POMDP under a fixed policy~\cite{KLC:planning}.} A decision tree (DT) is an easily visualisable data structure in which leaves represents actions prescribed by the policy and branching in internal nodes represents decisions that the policy makes in order to select a suitable action. To obtain a DT-representation from the corresponding table representation we utilize machine learning techniques for learning DTs. There advantage of this approach is that learning algorithms are often able to exploit the structure of the model, identify the crucial decisions made by the policy and encode \emph{only these} decisions in the DT. This typically results in much more succinct representation without significant loss of the policy's performance. To support this claim, we present experimental results on learning DT-represented policies using several well-known learning tools. As discussed in the previous paragraph, our approach can be seen as a generic technique for POMDPs which is particularly apt for use in the presence of energy constraints.

\section{Related Work} 
Our model of POMDPs with energy constraints, where the goal is to optimize the expected total cost while ensuring constraints on resource consumption, resembles the standard framework of constrained POMDPs~\cite{IMB:constrained-POMDP,UH:constrained-POMDP,KLKP:constrained-POMDP} a generic framework for enforcing constraints in POMDPs which has received considerable attention in various application domains~\cite{WJC:constrained-POMDP,JXWL:constrained-POMDP,UH:constrained-POMDP-decentralized} (see also~\cite{Altman:constrained-mdp} for related concepts in the setting of perfectly observable MPDs). 
The crucial difference between constrained POMDPs and our energy constraints is that the constraints in constrained POMDPs are soft, i.e. they are bounds on the \emph{expectation} of some quantity, while we require that the resource level stays always positive (not just on average) in all runs (see also a discussion in Section~\ref{sec:intro}).

As mentioned earlier, we extend the previous work on indefinite-horizon objective~\cite{BG09,Bertsekas95,CCGK:POMDP-cost} by adding energy constraints. Our notion of energy constraints is similar to the one used in verification, in particular to so called \emph{energy games and MDPs}~\cite{Chakrabarti2003a,CHD:energy-MDPs} and~\emph{consumption games}~\cite{BCKN:consumption-games}, although none of these concepts was considered in a partially observable setting so far.
DTs have already been successfully used to represent policies in verification of perfectly observable MDPs modelled in the well-known PRISM tool~\cite{BCCFK15}. For POMDPs, in~\cite{BP:succinct-pomdp} they consider a situation where the POMDP itself is encoded succinctly using DTs and similar structures, and they use this assumption to design a specific algorithm computing a desired policy (which itself is not encoded as a DT). In contrast, we assume that the model is given explicitly and use generic machine learning methods to infer succinct representations of policies. In~\cite{BG:DT-via-POMDP} they study relationship of DTs and POMDPs from an inverse perspective, POMDPs are used as a tool for learning decision trees from generic datasets. DTs were also used to represent policies in a reinforcement-learning setting~\cite{DGTB:DT-reinforcement-learning}, where the agent has no a priori model of the environment.

The need for succinct and efficient representation of policies motivated the study of finite-state controllers (FSCs) in POMDPs~\cite{GP:finite-state-incremental-POMDP,Hansen:finite-state-controllers-POMDP,MKKC:finite-state-controllers-POMDP,GPH:isomorph-free-FSC}. Intuitively, the approach is based on direct search for a small policy represented as a finite (possibly stochastic) transducer whose transitions are labelled by observations. In every step, the state of the transducer changes according to the transition function and latest observation received, and the controller then outputs an action to be performed based on the current state of the transducer. While this approach was shown to produce small and well-performing policies, we argue that our approach, while having similar goal, is conceptually different and offers an orthogonal set of advantages. The main difference is that FSCs represent a function whose domain are histories of actions and observations: each state of the finite transducer implicitly carries an information about the set of histories that lead the transducer to this state. The transducer thus captures an operational aspect of a policy, i.e. the way in which it is executed as a program. On the other hand, DTs represent functions whose domain is the set of beliefs: given a belief, a single root-leaf traversal of a DT is used to establish an action to be performed. Thus, DTs capture the logic of agent's decision in a concrete time instant; it is up to the agent to keep an (accurate or approximate) representation of its belief (which can be done using standard computations) and thus to take care of the history-dependent aspect of decision making. The latter approach more explicitly captures the decision making process as a human-like inference of suitable action from available information, and thus we believe that it provides better readability. Another advantage of our approach, which is validated by our experiments, is that the machine learning techniques we use are able to automatically identify the "most important" decisions that amount for the majority of optimization efforts. Finally, we show that FSCs can be prone to storing an unnecessary amount of data about resource levels. We further explain differences between the two formalism on a concrete example at the end of Section~\ref{sec:dt}.

One crucial difference between previous approaches to policy succinctness in both POMDP~\cite{GP:finite-state-incremental-POMDP} and other settings~\cite{DGTB:DT-reinforcement-learning} is that in previous work they concurrently optimize both the performance of a policy and its size, which requires dedicated algorithms, while we separate these tasks: first we search for a well-performing, though possibly "ugly" policy, and then learn its succinct representation (similar approach was used in~\cite{GPH:finite-controller-compilation}, where policies computed by point-based methods were "compiled" into FSCs). Thus, we present a framework for obtaining succinct representations in which various state-of-the art algorithms for POMDP solving and DT learning can be used. On the POMDP side, this allows us to keep up with advances in solving of large POMDPs. On the DT side, we can use well-developed machine learning tools that already offer a selection of methods for tree pruning and visualization, which is important for readability.

\input{example}

\section{Preliminaries}
\label{sec:prelim}

\paragraph{Notation} 
We use $\Nset_0,\Nset,\Zset$ to denote the sets of non-negative, positive, and all integers, respectively.
For $n\in \Nset$ we denote by $[n]$ the set $\{1,\dots,n\}$.
Let $X$, $Y$ be finite sets. For a function $f\colon X \rightarrow Y$ and sets $X' \subseteq X$, $Y'\subseteq Y$ we denote by $f(X')$ the image of $X'$ under $f$, i.e. the set $\{y \in Y\mid \exists x \in X' \colon f(x)=y\}$ and by $f^{-1}(Y')$ the pre-image of $Y'$ under $f$, i.e. the set $\{x\in X\mid f(x)\in Y'\}$. 
We denote by $\dist(X)$ the set of all probability distributions on $X$, i.e. of all functions $f\colon X \rightarrow [0,1]$ s.t. $\sum_{x\in X}f(x)=1$. For $f\in \dist(X)$ we denote by $\supp(f)$ the \emph{support} of $f$, i.e. the set $\{x\in X\mid f(x)>0\}$. A probability distribution $f$ is \emph{Dirac} if $|\supp(f)|=1$. An encoding size of an object $O$ (i.e. the number of bits needed to represent $O$) is denoted by $\size{O}$.

\paragraph{POMDPs} A \emph{Partially Observable Markov Decision Process (POMDP)} is a tuple $\pomdp = (\states,\acts,\tran,\obs,\obsfunc,\initdist)$ where: $\states$ is a finite set of \emph{states}; $\acts$ is a finite alphabet of \emph{actions}; $\tran\colon \states\times \acts \rightarrow \dist(\states)$ is a \emph{probabilistic transition function} assigning to every state-action pair a probability distribution over the successor states (i.e. $\delta(s,a)(s')$ denotes the probability of making a transition from $s$ to $s'$ under action $a$); $\obs$ is a finite set of \emph{observations}; $\obsfunc\colon \states \times \acts \rightarrow \dist(\obs)$ is a \emph{probabilistic observation function} assigning a probability distribution over observations to every state-action pair; and $\initdist$ is an \emph{initial probability distribution} over the states of $\pomdp$. 
We write $\delta(s'\mmid s,a)$ as a shorthand for $\delta(s,a)(s')$. 

\begin{remark}[Deterministic observation function]
\label{rem:det-obs}
We remark that deterministic observation functions of type $\obsfunc : S \rightarrow \obs$ are sufficient in POMDPs.
Informally, the probabilistic aspect of the observation function is captured in the 
transition function, and by enlarging the state space with the product with the observations,
we obtain an observation function only on states~\cite{CCGK:POMDP-cost}.
Thus in the sequel without loss of generality we will always consider observation function of type $\obsfunc : S \rightarrow \obs$ 
which greatly simplifies the notation. 
\end{remark}

\paragraph{Runs and Histories} A \emph{run} (finite or infinite) in a POMDP is an alternating sequence of states and actions $s_0,a_1,s_1,a_2,s_2,\dots$ such that $s_0 \in \supp(\initdist)$ and for every $i\geq 0$ it holds $\delta(s_{i+1}\mmid s_i,a_{i+1})>0$.  To a run $w = s_0,a_1,s_1,a_2,\dots$ we assign an \emph{observed run}, i.e. a corresponding observation-action sequence $\obsfunc(w)=\obsfunc(s_0),a_1,\obsfunc(s_1),a_2,\dots$.
A \emph{history} (finite or infinite) is an alternating sequence of observations and actions denoted as $\rho = z_0,a_1,z_1,a_2,z_2,\dots$, such that there exists a run $w$ for which $\rho = \obsfunc(w)$ (note that we already assume that function $\obsfunc$ is deterministic as noted in Remark~\ref{rem:det-obs}).

The \emph{length} of a finite run $w=s_0,a_1,\dots,s_k$ is the number $\length(w) =k$, i.e. the number of actions performed along $w$. The length of an infinite run is $\infty$, and the lengths of (finite or infinite) histories are defined likewise.
We denote by $\plays_{\pomdp}$ and $\fpaths_{\pomdp}$ the sets of all runs and finite histories in $\pomdp$, respectively. 

\paragraph{Policies} 
A \emph{policy} (or a \emph{policy}) in POMDP $\pomdp$ is a function $\sigma$ of type  $ \fpaths_\pomdp \rightarrow \dist(A)$. Intuitively, policies are abstractions of controllers for the system modelled by $\pomdp$:  the control is exerted by choosing a suitable action in every decision step, depending on the history of the system's evolution. A run $w=s_0,a_1,s_1,\dots$ \emph{conforms} to a policy $\sigma$ if for all $0\leq i <\length(w)$ the distribution $\sigma(\obsfunc(s_0,a_1,\dots,s_i))$ assigns positive probability to action $a_{i+1}$.


\paragraph{Semantics of POMDPs}
The behaviour of $\pomdp$ under a policy $\sigma$ can be intuitively described as follows: first, an initial state $s_0$ is sampled according the initial distribution $\initdist$. Then the system evolves in discrete steps. In a step $i\geq 0$, let $w_i=s_0,a_1,s_1,a_2,\dots a_{i},s_i$ be the current finite run, i.e. the sequence of traversed states and chosen actions up to the $i$-th step (we have $w_0 = s_0$). An action $a_i$ is sampled according to the distribution $\sigma(\obsfunc(w_i))$, and then a successor state $s_{i+1}$ is sampled according to the distribution $\delta(s_i,a_{i})$. In the next step the same procedure is performed with run $w_{i+1}=s_0,a_1,s_1,a_2,\dots a_{i+1},s_{i+1}$, etc. The process evolves in this manner \emph{ad infinitum}. This intuitive description can be formalized by constructing a suitable probability measure $\probm^{\sigma}$ assigning probabilities to sets of infinite runs in $\pomdp$. The construction of $\probm^{\sigma}$ is standard~\cite{Billingsley:book}. We denote by $\E^{\sigma}$ the expected value operator induced by $\probm^{\sigma}$.

\paragraph{Objectives}
An objective is a mathematical formalization of a desired behaviour of a system modeled by a POMDP. In this paper we study POMDPs that combine \emph{reachability}, \emph{stochastic shortest path}, and \emph{energy} objectives.

\begin{compactitem}
\item A \emph{reachability} objective is given by a set $\target \subseteq \states$ of \emph{target} states. A run $s_0,a_1,s_1,\dots$ satisfies such an objective if it eventually reaches a state from $T$, i.e. if $s_i \in T$ for some $i\geq 0$. We denote by $\Reach_{\target}$ the set of all infinite runs that satisfy a reachability objective with target set $\target$.

\item A \emph{total cost} objective is given by a tuple $(\target,\cost)$, where $T$ is again a set of target states and $\cost\colon \states\times\acts \rightarrow \Nset$ is a \emph{cost function} assigning a positive integer cost to every state-action pair. Total cost is a quantitative objective, i.e. instead of saying that a run satisfies the objective or not, we measure the "quality" of a run by assigning a number to it. Here we assign to an infinite run $w=s_0,a_1,s_1,\dots$ its total cost $\tot_\target^\cost(w)=\sum_{i=1}^{m} \cost(s_{i-1},a_i),$ where $m = \inf\{j \geq 0 \mid s_j \in T\}$. (We stipulate that an empty sum equals zero. Note that if $m=\infty$, then $\tot_\target^\cost(w)=\infty$.)

\item An \emph{energy} objective is given by a tuple $(\enrg,\ca,\target)$, where $\enrg\colon \acts \times \obs \rightarrow \Zset$ is a function assigning a \emph{resource change} to every action-observation pair, $\ca\in \Nset_0$ is a non-negative \emph{capacity}, and $\target$ is a set of target states. For $(s,a,n)\in \states\times\acts\times[\ca]$ we define a \emph{one-step resource update} $\enup^{\ca}(s,a,n)=\min\{\ca,n+\enrg(a,\obsfunc(s))\}$. For a run $w=s_0,a_1,s_1,\dots$ we put an \emph{energy level} after $i\in \Nset_0$ steps along $w$, where $0\leq i \leq \length(w)$, to be a number $\elev_{\enrg}^{\ca}(w,i)$ defined inductively as follows: $\elev_{\enrg}^{\ca}(w,0)=\ca$ and for $i\geq 1$ we put $\elev_{\enrg}^{\ca}(w,i)=\enup^{\ca}(s_{i-1},a_i,\elev_{\enrg}^{\ca}(w,i-1))$. In other words, we assume that the resource level is initially at full capacity and is then changed by performing various actions. Should the resource level rise above $\ca$, the excess amount is immediately discarded. An infinite run $w = s_0,a_1,s_1,\dots$ satisfies an energy objective given by $(\enrg,\ca,\target)$ if $\elev^{\ca}_{\enrg}(w,i)> 0$ for all $0 < i \leq m$, where $m = \inf\{j \geq 0 \mid s_j \in T\}$. We denote the set of all such satisfying infinite runs by $\Safe^{\ca}_{\enrg,T}$.
\end{compactitem}

\begin{remark}
Note that per our definition the resource level at every step is perfectly observable. This is a reasonable assumption whenever the modelled energy-constrained agent is equipped with sufficiently precise charge/fuel sensors. In this our model resembles \emph{mixed-observability} POMDPs~\cite{OPHL:mixed-observability,A-LTBC:mixed-observability}, and indeed in the next section we will present a transformation of POMDPs with energy constraints into standard POMDPs in which resource levels are a fully observable component of each state. However, mixed observability is used to enhance the performance of exact and point-based algorithms, while we aim for solution via simulation-based techniques, namely RTDP-Bel. Since, in the words of~\cite{A-LTBC:mixed-observability}, online techniques cannot be probably adapted to benefit from mixed observability, we stick to standard POMDP formulations. For further applications of mixed observability, see, e.g.~\cite{CCMNSB:mixed-observability,CMO:mixed-obs-robot-teams}.
\end{remark}

\paragraph{Computational Tasks}
Given a POMDP $\pomdp$, a set of states~$\target$, an resource change function $\enrg$, and a capacity $\ca$, we define 
the set of energy-safe policies $\energySafe^\pomdp_\target(\enrg,\ca)$. A policy $\straa$ belongs to the set $\energySafe^\pomdp_\target(\enrg,\ca)$
if for all infinite runs~$w$ conforming to policy $\straa$ we have $w \in \Safe^{\ca}_{\enrg,T}$.
Given a policy $\straa \in \energySafe^\pomdp_\target(\enrg,\ca)$ we define the value of $\straa$ 
as the expectation $\val(\straa) = \E_{}^{\straa}[\tot_\target^\cost]$. 

We are interested in minimizing the expected cost till the target set $\target$ is reached while keeping the energy level positive, i.e.,
we are interested in approximating the following value:

$$\optCost = \inf_{\straa \in \energySafe^\pomdp_\target(\enrg,\ca)} \val(\straa)$$

We aim to solve the following computational problems:
\begin{compactenum}
\item The \emph{qualitative energy-reachability problem} asks whether $\optCost < \infty$.
\item The \emph{quantitative energy-reachability problem}  asks for a policy $\straa$ such that $\val(\straa)$ approximates the value $\optCost$.
\end{compactenum}
%

\begin{remark}
We remark about POMDPs without energy constraints and the restrictions of the cost function:
\begin{compactenum}
\item The problem of approximating optimal cost $\optCost$ in POMDPs without energy constraints for positive costs was shown to be decidable in~\cite{CCGK:POMDP-cost}.
\item The problem of approximating optimal cost $\optCost$ for general costs (positive and negative) was proved to be undecidable in~\cite{CCGK:POMDP-cost} already for POMDPs without
energy constraints.
\end{compactenum}
\end{remark}
As policies are per definition infinite objects, in both the qualitative and quantitative problems we aim to compute their finite representations. One of the primary aims of this paper is to address the efficiency of such representations, our goal being to find succinct and/or human-readable ways to encode the computed policies.

\section{Solving Energy-Reachability \\ Problems}
\label{sec:product}
For the rest of the section let us fix a POMDP $\pomdp$, target set~$\target$, functions $\cost$, $\enrg$ and a capacity $\ca$. We will evaluate the complexity of presented algorithms in terms of $\size{\pomdp}$ and $\size{\ca}$. We assume that $\ca$ is represented in binary, i.e. $\ca$ is at most exponential in $\size{\ca}$. 

To solve both types of energy-reachability problems we construct a \emph{product} POMDP $\pomdp_{\times}$ by encoding the resource levels in $\pomdp_{\times}$ directly into the states. Formally, $\pomdp_{\times}$ has a set of states $\states_{\times} = \states\times[\ca]\cup\{\bot\}$, where $\bot$ is a newly added \emph{sink} state, and the same set of actions as $\pomdp$. A transition function $\delta_{\times}$ of $\pomdp_{\times}$ is defined as follows: for all $(s,n)\in \states \times [\ca]$, all $a\in \acts$ s.t. $\enup^{\ca}(s,a,n)\geq 1$, and all $\tilde{s}\in S_{\times}$ we have
\[
\delta_{\times}(\tilde{s}\mmid (s,n),a) = \begin{cases}
\delta(s'\mmid s,a) & \text{if } \tilde{s} = (s',\enup^{\ca}(s,a,n)) \\
0 & \text{otherwise};
\end{cases}
\]
and for every other $\hat{s}\in \states_{\times}$, $a\in \acts$ the distribution $\delta_{\times}(\hat{s},a)$ is Dirac, assigning $1$ to state $\bot$. 
The set of observations in $\pomdp_\times$ is $\obs_{\times} = \obs \cup \{\bot\}$, and observation function $\obsfunc_\times$ is such that $\obsfunc_\times(s,n)=\obsfunc(s)$ for each $(s,n)\in \states\times [\ca]$ and $\obsfunc_\times(\bot)=\bot$. Finally, the initial distribution $\initdist_\times$ assigns to each tuple of the form $(s,\ca)$ probability $\initdist(s)$, and $0$ to all other states. We also extend the reachability and total cost objectives to $\pomdp_\times$ by defining a new target set $\target_\times = \target \times [\ca]$ and cost function $\cost_\times$ such that $\cost_\times((s,n),a)=\cost(s,a)$  and $\cost_\times(\bot,a)=1$, for every action~$a$.

It is straightforward to verify that POMDP $\pomdp_\times$ can be automatically constructed in time polynomial in $\size{\pomdp}$ and exponential in $\size{\ca}$.

There is a natural correspondence between runs, histories, and policies in $\pomdp$ and $\pomdp_\times$ which preserves the properties related to our objectives. In particular, for every policy $\sigma \in \energySafe^\pomdp_\target(\enrg,\ca)$ that almost surely reaches $\target$ one can construct a policy $\tilde{\sigma}$ in $\pomdp_\times$ such that $\tilde{\sigma}$ almost surely reaches $\target_\times$ and $\E^{\tilde{\sigma}}[\tot^{\cost_{\times}}_{\target_\times}]=\E^{\sigma}[\tot^{\cost_{}}_{\target}]$; and vice versa, any policy in $\pomdp_\times$ that almost surely reaches target, can be transformed into a policy of the same expected cost in $\energySafe^\pomdp_\target(\enrg,\ca)$. If the policy $\sigma$ to be transformed is finitely represented, the finite representation of the transformed policy $\tilde{\sigma}$ can be computed in time polynomial in $\size{\pomdp_\times}$ and $\size{\sigma}$. 

It follows that to solve the qualitative energy-reachability problem for $\pomdp$ it suffices to solve the qualitative reachability problem for $\pomdp_\times$, i.e. compute a policy $\tilde{\sigma}$ such that $\probm^{\tilde{\sigma}}(\Reach_{\target_\times})=1$. 
Algorithm solving the qualitative reachability problem based on \emph{belief supports} was presented in~\cite{CCGK:POMDP-cost,BGB:prob-omega}. We briefly recall the approach: A belief support of a finite history $\rho=z_0,a_1,z_1,\dots,z_n$ is a set $\belsup(\rho)$ of states in which the POMDP can be with positive probability after the sequence $\rho$ is observed, i.e. $\belsup(\rho)=\{s\in \states \mid \exists w = s_0,a_1,s_1...,s_n \in \plays_{\pomdp_\times} \colon \rho = \obsfunc(w) \wedge s = s_n\}$. 
The algorithm computes for each $U \in \belsupset(\pomdp_\times)=\{U \subseteq \states\mid \exists \rho \in \fpaths_{\pomdp_\times}\colon\belsup(\rho)=U\}$ a set of so-called \emph{allowed} actions in $U$: intuitively, action $a$ is allowed in $U$ if playing action $a$ in any situation where the observed finite history has belief support $U$ results into situation in which the target set can still be reached with probability 1 by some policy. One can show that if there is a state $s\in\supp(\lambda_0)$ such that $\belsup(\obsfunc(s))$ admits no allowed action, then no policy in $\pomdp_\times$ can reach the target almost surely. Otherwise the algorithm outputs a policy $\sigma_{\allow}$ which for each finite history $\rho$ plays all actions allowed in $\belsup(\rho)$ with uniform probability. It can be proves that for $\sigma_{\allow}$ it holds $\probm^{\sigma_\allow}(\Reach_{\target_\times})=1$.

The running time of the algorithm and the space needed to represent $\sigma_\allow$ is dominated by a polynomial in the the size of $\belsupset(\pomdp_\times)$, i.e. in the number of reachable belief supports. This number can be trivially bounded by $2^{|\states_\times|}$, which is an expression doubly exponential in $\size{\ca}$. However, from the construction of $\pomdp_\times$ we get the following improved bound:

\begin{lemma}
\label{lem:main}
It holds $|\belsupset(\pomdp_\times)| \leq 2^{|\states|\cdot \ca}$. 
\end{lemma}


As a consequence we get the following.

\begin{theorem}
\label{thm:qual-complexity}
The qualitative energy-reachability problem for POMDPs is \textsc{EXPTIME}-complete.
\end{theorem}

\begin{proof}
The upper bound follows from Lemma~\ref{lem:main}, and from the complexity of constructing $\pomdp_\times$ and translating its policies to $\pomdp$. The lower bound follows from \textsc{EXPTIME}-hardness of qualitative reachability in POMDPs~\cite{CDH10a}.
\end{proof}

To solve the quantitative energy-reachability problem, we again use an algorithm for POMDPs without energy, namely the one from~\cite{CCGK:POMDP-cost}, and apply it to $\pomdp_\times$. The algorithm, which assumes that the sets of allowed actions were already computed via the aforementioned method, finds a policy of small cost that almost surely reaches $\target_\times$, using a modified version of RTDP-Bel~\cite{BG09}. RTDP-Bel is an adaptation of the \emph{real-time dynamic programming} value iteration~\cite{BBS:rtdp} to POMDPs. It is an approximative method which does not guarantee convergence to optimum, but it is known to produce near-optimal policies on many instances where optimal costs can be computed using exact methods~\cite{CCGK:POMDP-cost}. Hence, results produced by RTDP-Bel are a useful yardstick against which policies obtained by other methods can be compared. Due to the absence of guarantees we do not investigate the theoretical complexity of the algorithm. Its experimental evaluation can be found in Section~\ref{sec:exp}.

The policy output by modified RTDP-Bel bases its decision in every step on the current \emph{belief}, i.e. the probability distribution over the set of states representing the likelihood of being in particular states given the current history of states and observations~\cite{KLC:planning}. As the space of beliefs is continuous, the policy operates on its discretized version, which allows it to be represented by a finite table storing one action per discretized belief. Not all beliefs have to be stored in the table: RTDP-Bel can converge to optimum without considering all reachable beliefs. Still, memory required to store the table might be too large  for the policy to be understandable. In the next section we present a framework for converting table-represented policies into a more succinct and human-readable form.


%


\newcommand{\belief}{b}
\section{Succinct Representation of Policies}
\label{sec:dt}
A policy in $\pomdp_\times$ computed by the RTDP-Bel is a function which to every belief assigns an action to be taken. Formally, a belief is a probability distribution $\belief$ on $\states_{\times}$ such that $\supp(\belief)\subseteq \obsfunc_{\times}^{-1}(z)$ for some observation $z$.

As indicated above, RTDP-Bel considers only discretized beliefs, that is beliefs whose probabilities are rounded to a finite mesh. For technical reasons, the RTDP-Bel represents such discretized beliefs as vectors of non-negative integers from an interval $[0,B]$ where $B$ is a bound which determines the precision of the approximation. 
%

Thus each belief in $\pomdp_{\times}$ can be represented as a vector $\belief\in \Zset^{\states+1}$ whose first $|\states|$ components are integers from $[0,B]$, and whose last component is in $[\ca]$.
Given such a belief $\belief$, the true probability of being in a state $s$ with the energy $n$ is (approximately) equal to $\belief^s/B$, where $\belief^s$ is the component corresponding to state $s$.

For simplicity, we assume that actions are named in such a way that $\acts=\{0,1,\ldots,k_{\acts}\}$.

\subsection{Decision Trees}
\label{subsec:dt}
There are numerous possibilities of succinctly representing sets of vectors of numbers (and functions on such sets) in a human readable form. One of the most popular formalisms suitable for this purpose are decision trees (DT~see~\cite{Quin86,Mitch97}). We use DTs to represent functions of beliefs in POMDPs.
For convenience, we follow closely the definition of DT used in~\cite{BCCFK15}. 
Let $V=\{v_1,\ldots,v_d\}$ be a set of variable names.
\begin{definition}
A \emph{decision tree} over the set of variables $V$ is a tuple $\mathcal{T}=(Tr,\rho,\theta)$ where $Tr$ is a finite rooted binary (ordered) tree with a set of inner nodes $N$ and a set of leaves $L$, $\rho$ assigns to every inner node a predicate of the form $[v_i\sim \mathit{const}]$ where $v_i\in V$, $\mathit{const}\in \Zset$, $\mathord{\sim}\in\{\leq,<,\geq,>,=\}$, and $\theta$ assigns to every leaf a non-negative integer.
\end{definition}
A DT $\mathcal{T}$ over $V$ determines a function $f:\Zset^d\rightarrow \Zset$ as follows:
For a vector $\vec{v}=(\bar v_1,\ldots,\bar v_n)\in \Zset^d$, we find a path $p$ from the root to a leaf $\ell$ such that for each inner node $n$ on the path, the predicate $\rho(n)$ is satisfied by substitution $v_i=\bar v_i$ iff the~first child of $n$ is on $p$. Then we put $f(\vec{v})=f(\bar v_1,\ldots,\bar v_n)=\theta(\ell)$.
%
In our setting the set of variable names is chosen so as to suitably characterize the current belief. Typically, one can put $V=S\cup\{\mathit{Energy}\}$, although different sets can be used as well. The domain of values assigned to leaves is the set of actions $\acts=\{0,1,\ldots,k_{\acts}\}$.
\subsubsection{Training DT.}
We describe the process of \emph{learning a training set}, which can also be understood as storing the input/output behaviour of a function described by data.
Assume that we are given a training sequence $\tau=(\vec v^1,f_1),\ldots,(\vec v^k,f_k)$ (repetitions allowed!) that specifies the desired input/output behaviour, i.e. each $\vec v^i=(v^i_1,\ldots,v^i_n)\in \Zset^d$ is a training input and $f_i\in \Zset$ is the expected output. The goal is to learn a DT which exhibits the input/output behaviour prescribed by the training sequence.

A standard process of learning according to the algorithm ID3 \cite{Quin86,Mitch97} proceeds as follows:
                \begin{compactenum}
                \item Start with a single node (root), and assign to it the whole training sequence.
                \item For a node $n$ with a sequence $\tau=(\vec v^1,f_1),\ldots,(\vec v^k,f_k)$,
                \begin{compactenum}
                \item if all training examples in $\tau$ have the same expected output value (i.e. there is $x$ such that $f_i=x$ for all $i$), set $\theta(n)=x$ and stop;
                \item otherwise,
                        \begin{compactitem}
                        \item choose a predicate with the ``highest information gain'' (with lowest entropy, see e.g.~\cite[Sections 3.4.1, 3.7.2]{Mitch97}),
                        \item split $\tau$ (according to the inputs) into sequences satisfying and not satisfying the predicate, assign them to the first and the second child, respectively,
                        \item go to step 2 for each child.
	                 \end{compactitem}
                \end{compactenum}
                \end{compactenum}
Intuitively, the predicate with the highest information gain tends to make a clear cut among the classes.

In addition, the final tree can be \emph{pruned}. This means that some leaves are merged, resulting in a smaller tree at the cost of some imprecision of storing.
The pruning phase is quite sophisticated, hence for the sake of simplicity and brevity, we omit the details here.
We use the standard C4.5 algorithm and refer to~\cite{Quinlan93,Mitch97}.
In Section~\ref{sec:exp}, we comment on effects of parameters used in pruning. We also use the CART algorithm~\cite{BreimanFOS84} with so called Gini index instead of the information gain to select the best splits (there are also differences in pruning).

\subsection{Learning a DT Policy}
\label{subsec:learn}
Our goal is to train a DT to succinctly represent a policy computed by  RTDP-Bel. We use RTDP-Bel to generate a training set using the following procedure:
\begin{compactitem}
\item Compute a policy $\sigma$ using RTDP-Bel that solves the quantitative energy-reachability problem.
\item Run a specified number $m$ of simulations of $\sigma$, each of a fixed length $\ell$.
In every step of each simulation produce a new training instance $(\belief,\sigma(\belief))$ where
	\begin{compactitem}
	\item $\belief$ is the current belief,
	\item $\sigma(\belief)$ is the action chosen by $\sigma$ in the current step.
	\end{compactitem}
\end{compactitem}
The above procedure generates a training sequence of pairs $\tau=(b_1,\sigma(b_1)),\ldots,(b_k,\sigma(b_k))$ where $k=m\cdot \ell$. We feed this sequence into a learning algorithm for DT and obtain a decision tree approximating behaviour of $\sigma$ on beliefs visited by the simulations. 

Producing $\tau$ using RTDP-Bel should serve as a tool for detecting the most important decisions of $\sigma$. The intuition is that whenever a decision in a belief is made repeatedly in many simulations, it is worth remembering in the~DT.

\paragraph{Execution of a DT Policy.} The agent maintains the current belief over the state-space. The decision
tree $\mathcal{T}$ represents a function from beliefs to recommended actions as described in Section~\ref{subsec:dt}. In every step, the agent computes the value of the function for the current belief. Whenever a non-allowed action is recommended by the decision tree, we detect this
and play all allowed actions uniformly at random, taking advantage of the precomputed set of allowed actions. Note that similar situation may also arise when executing a policy computed by RTDP-Bel: such a policy does not typically store decisions for all beliefs (see Section~\ref{sec:product}) and hence it might happen during its execution that no entry for the current belief can be found. Whatever action the implementation chooses in such a situation, the action must be allowed in the current belief-support, because after playing a non-allowed action, the set of target states would not be reached
with probability~$1$. It is thus reasonable to assume that the precomputed information on allowed actions must be stored when executing any of the two types of policies. Hence, we view  the difference between the size of the RTDP-Bel policy (i.e. the number of entries in its table) and the size of the DT learned from this policy (i.e. the number of its nodes) to be the primary measure of how much succinctness can be achieved by using decision trees.

\paragraph{Comparison with Finite-state Controllers.}
We illustrate the conceptual differences between finite-state controllers and decision trees on a toy POMDP example.
Figure~\ref{fig:tiger} depicts the \emph{Energy-constrained Tiger POMDP} which is an extension of the famous Tiger problem introduced in~\cite{kaelbling1998planning}. Imagine an agent standing in front of two closed doors. Behind one of the doors is a tiger and behind the other is a treasure. If the agent opens the door with the tiger a huge cost is received. In all other situations no cost is received. After opening any of the doors the POMDP reaches a terminal configuration. The agent can also listen, in order to gain some information about the location of the tiger. Unfortunately, listening is not entirely accurate. There is a $15\%$ chance that the agent will hear a tiger behind the left-hand door when the tiger is really behind the right-hand door, and vice versa. 
We consider two cases in the first one the listening action has no resource consumption and in the second case it has a resource consumption associated with the action. In the second case if the agent is running out of the resource, it may decide to recharge in a neighbouring location.

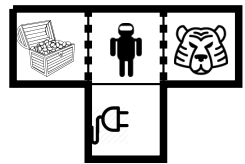
\begin{figure}[t]
\begin{center}

\pgfdeclareimage[width=0.9cm]{treasure}{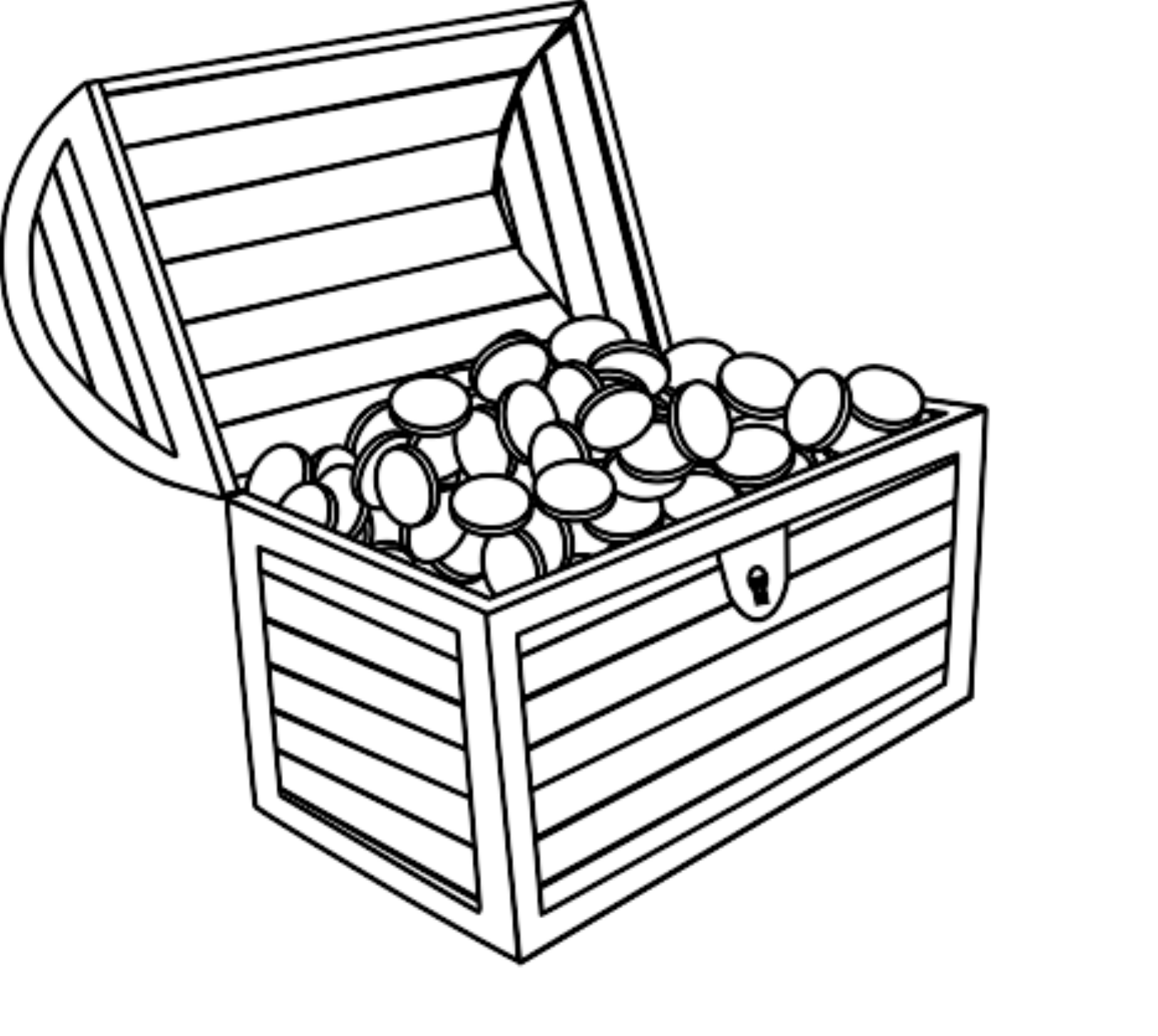}
\pgfdeclareimage[width=0.7cm]{tiger}{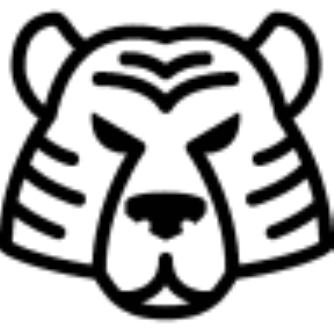}
\pgfdeclareimage[width=0.5cm]{plug}{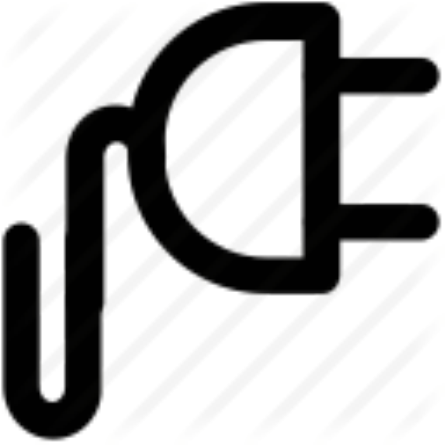}
\pgfdeclareimage[width=0.7cm]{robot}{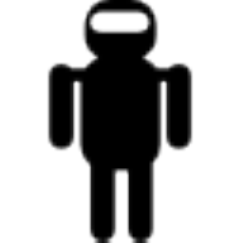}

\usetikzlibrary{positioning}

\begin{tikzpicture}
  \tikzstyle{every node}=[font=\large]
  \node[] at (0.5,4.45) {\pgfuseimage{treasure}};
  \node[] at (2.5,4.45) {\pgfuseimage{tiger}};
  \node[] at (1.28,3.45) {\pgfuseimage{plug}};
  \node[] at (1.5,4.45) {\pgfuseimage{robot}};
  \draw[line width=0.1cm] (0,5) -- (3,5) -- (3,4) -- (2,4) -- (2,3) -- (1,3) -- (1,4) -- (0,4) -- (0,5);
  \draw[line width=0.1cm, dashed] (1,5) -- (1,4);
  \draw[line width=0.1cm, dashed] (2,5) -- (2,4);
  \draw[thick] (0,5)--(3,5);
  \draw[thick] (0,4)--(0,5);
  \draw[thick] (0,4)--(3,4);
  \draw[thick] (3,4)--(3,5);
  \draw[thick] (2,3)--(2,5);
  \draw[thick] (1,3)--(1,5);
  \draw[thick] (1,3)--(2,3);

\end{tikzpicture}
\caption{Energy-constrained Tiger Problem.}
\label{fig:tiger}
\end{center}
\end{figure}

In the case without energy consumption, the optimal cost is 0, as for every $\varepsilon \geq 0$ there is a number $T$ such that after performing $L$ listening actions the robot builds enough confidence about the position of the tiger so as to choose the good door with probability at least $1-\varepsilon$. At the same time, $L$ listening actions are necessary to build such a confidence, and $L \rightarrow \infty$ as $\varepsilon \rightarrow 0$. To represent a policy that waits until enough confidence is built and then decides we can employ both FSCs and DTs. In the case of FSC the corresponding finite transducer needs to have number of states proportional to $L$, as it needs to count the number of listening actions so far as well as results of these actions. In a DT case the policy can be represented using a DT with 5 nodes, which is depicted in Figure~\ref{fig:dt}. We use two variables $\mathsf{Tiger-Left}$ and $\mathsf{Tiger-Right}$, that represent the probability that the tiger is behind the respective door. In each step the robot straightforwardly extracts these probabilities from its current belief and uses the decision tree to select an appropriate action. (A solid line to a successor is taken if the condition inside an inner node is satisfied, otherwise the dashed line is taken. The parameter $\delta$ represents the confidence level sufficient for door selection.) In particular, the size of the tree is independent of $\varepsilon$.

In the case with energy consumption, the FSC either needs to store the information on the current resource level in the transducer's state, or, if the resource levels are contained in observations (as is the case in our reduction in the previous section), the FSC needs to have, in each of its states, at least one outgoing transition per each energy level (as there must be at least one outgoing transition for each possible observation). Thus, the size of such a controller is proportional also to the capacity $\ca$. On the other hand, using a DT it suffices to slightly modify the tree in Figure~\ref{fig:dt}. We need to use an additional variable representing the current resource level (which we assume is precisely known to the robot), and add a new root node which tests whether the energy is greater than $1$: if it is, the robot goes to a sub-tree depicted in Figure~\ref{fig:dt}, otherwise it recharges in the neighbouring location.

\begin{figure}
\begin{center}
\resizebox{0.5\linewidth}{!}{
\begin{tikzpicture}[node distance = .9cm]
  \node (a1) [draw,rectangle,rounded corners=2pt] {$\mathsf{Tiger-Left} \geq \delta$};

  \node (a2) [below left of=a1,draw,rectangle,rounded corners=2pt,yshift=-0.5cm,align=left,xshift=-2cm] {$\mathsf{Open-Door-Left}$};
  \node (bad1) [below right of=a1,draw,rectangle,rounded corners=2pt,xshift=2cm,yshift=-0.5cm] {$\mathsf{Tiger-Right} \geq \delta$};
  \node (a3) [below left of=bad1,draw,rectangle,rounded corners=2pt,yshift=-0.5cm,align=left,xshift=-3cm] {$\mathsf{Open-Door-Right}$};
  \node (bad2) [below right of=bad1,draw,rectangle,rounded corners=2pt,yshift=-0.5cm,xshift=-1cm] {$\mathsf{Listen}$};
  
  \draw [-,dashed] (a1) to (bad1);
  \draw [-] (a1) to (a2);
  \draw [-,dashed] (bad1) to (bad2);
  \draw [-] (bad1) to (a3);
\end{tikzpicture}
}
\caption{Decision Tree for the Tiger Problem.}
\label{fig:dt}
\end{center}
\end{figure}
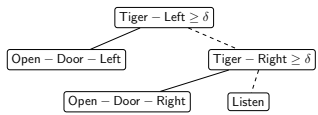

\input{experiments}

\section{Conclusion}
In this work we have considered POMDPs with a set of target states, positive integer costs associated with every transition, and resource levels. We present
a novel algorithm for solving POMDPs enhanced with resource levels based on the
existing POMDP solvers and the RTDP method. We consider three different approaches to obtain succinct and human-readable policies. On two scalable domains from the existing literature we present succinct policies that perform only slightly worse than the optimal policy, while being significantly smaller. 


\bibliography{petr-pomdp-energy,tomas-DT,energy}
\bibliographystyle{plain} 
%
\end{document}

%% file: example.tex


%% file: experiments.tex
\section{Experimental results}
\label{sec:exp}
We have extended the POMDP file format introduced in~\cite{POMDPSolve} with the constructs necessary to model resource consumption.
We have implemented Algorithm~\ref{alg:1} 
that implements the product construction of Section~\ref{sec:product}.
We use a modified version
of the RTDP-Bel POMDP solver~\cite{BG09} to solve these energy-constrained POMDPs and generate training data for machine learning tools (see the previous section). 

\begin{table*}[!t]
\centering
\resizebox{17cm}{!}{
\begin{tabular}{|c|c|c|c|c|cc|cc|cc|cc|}
\toprule
\multirow{2}{*}{Name} & \multirow{2}{*}{\: $\vert \states \vert $\:} & \multirow{2}{*}{\: $\ca$\:} & \multirow{2}{*}{\: $\vert \states_\times \vert$ } & \multirow{2}{*}{\: $\val(\straa_{all})$} & \multicolumn{2}{c|}{\textbf{RTDP-Bel}} & \multicolumn{2}{c|}{\textbf{Weka}} & \multicolumn{2}{c|}{\textbf{RPart}} & \multicolumn{2}{c|}{\textbf{Tree}}\\
& & & & & Mem. size & $\val$ & DT. size & $\val$ & DT. size & $\val$ &DT. size & $\val$\\
\midrule
Hallway5x5 & 51 & 10 & 197 & {\it 568.001} & 216 & 10.998 & -- & -- & -- & -- & -- & -- \\ 
Hallway6x6 & 83 & 10 & 443 & {\it 862.746} & 314 & 15.143 & 1 & 32.390 & \textbf{17} & \textbf{22.143} & 15 & 31.596\\
Hallway8x8 & 155 & 10 & 961 & {\it 954.160} & 537 & 21.020 & 5 & 44.399 & \textbf{21} & \textbf{26.396} & 5 & 52.689 \\
Hallway10x10 & 259 & 15 & 2437 & {\it 984.647} & 2441 & 23.027 & -- & -- & \textbf{25} & \textbf{60.643} & 55 & {\it 573.983} \\
Hallway13x13 & 435 & 10 & 2852 & {\it 958.284} & 1684 & 31.492 & -- & -- & -- & -- & \textbf{51} & \textbf{58.395} \\
Hallway17x17 & 835 & 10 & 5163 & {\it 797.728} & 2777 & 26.904 & 15 & 46.627 & 7 & 41.820 & \textbf{9} & \textbf{40.480}\\
Hallway17x17 & 835 & 15 & 8882 & {\it 716.796} & 6305 & 27.159 & 9 & 59.022 & 7 & 51.880 & \textbf{51} & \textbf{38.400}\\
Hallway19x19 & 1103 & 10 & 8891 & {\it 883.223} & 6801 & 30.131 & 15 & 148.184 & 5 & 142.537 & \textbf{3} & \textbf{44.614} \\
\midrule
RockSample[3,4] & 435 & 7 & 2403 & 12.891 & 1257 & 6.146 & \textbf{31} & \textbf{8.52} & 11 & 12.948 & 1 & 13.638\\
RockSample[3,5] & 1011 & 7 & 5425 & 12.412 & 395 & 6.920 & 19 & 8.275 & 13 & 12.405 & \textbf{1} &  \textbf{8.275}\\
RockSample[4,4] & 771 & 7 & 4297 & 14.852 & 570 & 9.408 & \textbf{25} & \textbf{14.285} & 11 & 14.964 & 1 & 15.644 \\
RockSample[5,4] & 1803 & 6 & 5679 & 16.290 & 421 & 9.578 & 25 & 17.192 & \textbf{11} & \textbf{17.044} & 1 & 17.192\\
\bottomrule
\end{tabular}
}
\caption{Experimental results}
\label{tab:results}
\end{table*}

\begin{algorithm}[!ht]
  \caption{}
  \label{alg:1}
\begin{algorithmic}[]
\State \textbf{Input:} POMDP $\pomdp$ with energy-reachability objective
\State \textbf{Output:} A succinct policy $\straa$ if there exists one
\State
\State $\pomdp_{\times} \quad \gets \quad \mathsf{constructProduct}(\pomdp)$ 
\Comment Section~\ref{sec:product}
\State $\straa_r \hspace{1.8em}  \gets \quad \mathsf{RTDP\!\!-\!\!Bel}(\pomdp_{\times})$
\State $\tau \hspace{2.3em} \gets \quad \mathsf{trainingData(\straa_r)}$
\Comment Section~\ref{sec:dt}
\State $\straa \hspace{2.3em} \gets \quad \mathsf{trainTree(\tau)}$
\Comment Section~\ref{sec:dt}
\State \textbf{return:} $\straa$
\end{algorithmic}
\end{algorithm}

\smallskip\noindent{\em Generating Training Data.}
In all the examples we consider we choose appropriate variables to characterize the beliefs , e.g. in grid-like environments we have variables for the current row, column, and current resource level (recall that resource level is perfectly observable). We parse the policy produced by RTDP-Bel to obtain a training sequence $(\vec{v}_1,a_1),(\vec{v}_2,a_2),\dots$ where each $\vec{v}_i$ is a vector of variable values characterizing a belief to which the RTDP-Bel policy assigns action $a_i$.

\smallskip\noindent{\em Decision Tree Learning.}
We use three different methods to obtain decision trees for our examples:
\begin{compactenum}
\item In the first scenario decision trees are constructed using the Weka machine learning package~\cite{Hall2009}. The Weka suite offers various decision tree classifiers. 
We use the J48 classifier, which is an implementation of the C4.5 algorithm \cite{Quinlan93};
\item We use the package {\it rpart}~\cite{rpart} from R~\cite{R} which implements the CART algorithms of~\cite{BreimanFOS84}. We use the Gini index as default for selecting the best splits. We have experimented with tree pruning  using complexity parameters~(see~\cite{rpart}).
\item Finally, we constructed trees using the package {\it tree}~\cite{tree} of R which implements algorithms of~\cite{Ripley96}. In this case, the default measure for selecting splits is the deviance (see~\cite{Ripley96}). Note that the results of tree and rpart packages are usually very different.
\end{compactenum}

\vspace{0.4em}
We experimented on two well-known examples of POMDPs naturally extended with resource levels. Typically, we assume that all the actions of the agent
decrease the resource level and there are specific recharging locations in the area that restore the resource level to the maximum capacity.
The POMDP examples we considered are the following:
(A)~We consider the Hallway example from~\cite{LCK95,S04,SS04,BG09}.
(B)~We consider the RockSample example from~\cite{BG09,SS04}.

\subsection{Succinct Policies: Example}
\label{subsec:succ}
In this part we discuss an example of a succinct human readable policy for the
Hallway examples

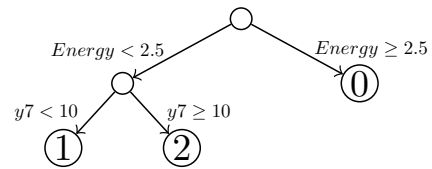
\begin{figure}[!ht]
\centering
\resizebox{6cm}{!}{

\begin{tikzpicture}[x=1.1cm, y=.3cm]

\node[node] (0) at (0,0) {};
\node[node] (00) at (-1,-2) {};
\node[node] (000) at (-1.5,-4) {$1$};
\node[node] (001) at (-.5,-4) {$2$};
\node[node] (01) at (1,-2) {$0$};

\draw[->] (0) to node[midway,left=0.1cm,scale=.5] {$Energy < 2.5$} (00);
\draw[->] (0) to node[midway,right=0.1cm,scale=.5] {$Energy \geq 2.5$} (01);
\draw[->] (00) to node[midway,right=0.1cm,scale=.5] {$y7 \geq 10$} (001);
\draw[->] (00) to node[midway,left=0.1cm,scale=.5] {$y7 < 10$} (000);


\end{tikzpicture}

 }
\caption{policy for Hallway on an 8x8 grid with $\ca = 10$.}
\label{fig:policy}
\end{figure}

\begin{example}
\label{ex:tree}
Figure~\ref{fig:policy} shows a decision tree computed via Tree package for an instance of a Hallway example, which models a robot navigation in a maze. We use variable names $x_1,\dots,x_{8},y_1,\dots,y_{8},\mathit{Energy}$, where the values $x_n$, $y_n$ represent the probability that the $x$ or $y$ coordinate, respectively, of the robot is equal to $n$ (we have $B=20$, i.e. value $20$ represents probability 1). For better readability, predicates are not contained within nodes, they label edges instead. To execute a policy represented by the tree, the robot looks, in every step, on its current belief. If the current resource level is at least $3$ (recall that resource levels are integers), it performs action 0 ("move forward"). Otherwise the robot checks the probability of the current $y$ coordinate being 7. If it is smaller than $\frac{1}{2}$, the robot turns left (action 1), otherwise it turns right (action 2). If at any point the action recommended by the tree is not allowed, the robot chooses an allowed action uniformly at random. 
\end{example}

\subsection{Discussion on Experimental results.}
We present the results of our approach in Table~\ref{tab:results}. Every entry contains the following information:
(i)~the name of the benchmark; (ii)~the size of the state space; (iii)~the maximum resource level $\ca$; (iv)~the size of the the product state space $\states_\times$ after a preprocessing step which removes unreachable states; 
(v)~the value of the policy $\straa_{all}$ that plays all allowed actions uniformly at random; (vi)~RTDP-Bel entries, that present the size of the computed policy and
its corresponding value; (vii)~for Weka, RPart, Tree we present the size of the computed decision tree and the value of the corresponding policy.
The entries labelled with "-" did not have a run that reaches the set of target states $T$, entries in italics do not reach the set of target states $T$ with all the runs within the run cut-off length of $1000$, i.e., the expected cost for the policies may be higher. The bold-faced entries present the best result among the three considered DT-learning tools.

The entries show that removing the unreachable states from the product POMDP $\pomdp_\times$ is efficient and allows scaling to larger instances, e.g., a naive product construction in entry Hallway 10x10 would yield a product POMDP with $3885$ states compared to the $2437$ reachable states.

In most of the cases our approach succeeded and computed a significantly smaller policy than the standard policy computed by RTDP-Bel. The computed succinct policies usually perform slightly
worse than the optimal explicit policy, however still overwhelmingly outperform the naive policy $\straa_{all}$ that plays all actions uniformly at random. For instance, the tree for Hallway8x8 presented in Example~\ref{ex:tree} identifies two crucial decisions (one based on current resource  level, and one on crossing a certain "latitude") with which the policy already significantly outperforms $\sigma_{all}$ and achieves performance relatively close to the RTDP-Bel policy of size 537. Interestingly, this suggests that POMDPs with SSP objectives exhibit a phenomenon known as \emph{Pareto principle}~\cite{Newman:power-laws}, where a small fraction of decisions accounts for majority of optimization effort. Moreover, the learning techniques we used are typically able to identify such decisions. Even for RockSample examples, where there is not much room for imprecision ($\sigma_{all}$ incurs only the double the cost incurred by the corresponding RTDP-Bel policy) the DT policies performed relatively well on some instances.

Interestingly,
no approach significantly outperforms any other and for each of the three considered approaches there are entries where it dominates. On a fraction of examples, such as Hallway 5x5 or on the large RockSample instances, the learned policies do not perform well. One way of improving the performance of DT policies would be to use more expressive variants of DTs, such as \emph{linear DTs}, that can capture general linear dependencies between variables. Testing this conjecture would require extending the aforementioned DT-learning tools with the capability of learning linear DTs, which we deem to be a viable direction of future work.  